\documentclass{article}

\usepackage[final,nonatbib]{nips_2017}

\usepackage[utf8]{inputenc} 
\usepackage[T1]{fontenc}    
\usepackage{hyperref}       
\usepackage{url}            
\usepackage{booktabs}       
\usepackage{amsfonts}       
\usepackage{nicefrac}       
\usepackage{microtype}      

\usepackage{bm}
\usepackage{color}
\usepackage{amsmath,amssymb,amsthm}
\usepackage{graphicx}

\usepackage{hyperref}
\definecolor{myred}{rgb}{0.77, 0.0, 0.1}
\definecolor{newgreen}{RGB}{0,153,0}
\definecolor{myturq}{rgb}{0.1, 0.7, 0.7}
\hypersetup{
    colorlinks,
    linkcolor=blue, 
    citecolor=newgreen, 
    linktoc=all 
  }

\usepackage{mathrsfs}
\usepackage{relsize}
\usepackage{algorithmic}
\usepackage{algorithm}
\usepackage{appendix}

\renewcommand{\leq}{\leqslant}
\renewcommand{\geq}{\geqslant}

\newcommand{\eps}{\varepsilon}

\newcommand{\cond}{\,|\,}

\newcommand{\vol}{\mathop{\mathrm{vol}}}

\newcommand{\wt}{\tilde}
\newcommand{\wh}{\widehat}
\newcommand{\pp}{\, : \; }

\newcommand{\N}{\mathbf N}

\newcommand{\R}{\mathbf R}

\newcommand{\D}{\mathscr D}

\newcommand{\E}{\mathbb E}
\renewcommand{\P}{\mathbb P}
\newcommand{\Var}{\mathrm{Var}}

\newcommand{\ie}{\textit{i.e.} }

\newcommand{\MP}{\mathop{\mathrm{MP}}} 
\newcommand{\splits}{\Sigma} 
\renewcommand{\split}{\sigma} 
\newcommand{\node}{\eta}
\newcommand{\nodes}{\mathcal N} 
\newcommand{\inodes}{\mathcal N^{\circ}} 
\newcommand{\leaves}{\mathcal L} 
\newcommand{\leaf}{\phi} 
\newcommand{\diam}{\mathop{\mathrm{diam}}}

\newcommand{\Exp}{\mathrm{Exp}}
\newcommand{\cell}{A}

\newcommand{\lchild}[1]{\mathtt{left}( {#1} )} 
\newcommand{\rchild}[1]{\mathtt{right}( {#1} )} 

\newcommand{\cut}{\nu} 

\newcommand{\cv}{\mathscr{C}}
\newcommand{\range}{\mathtt{range}}

\newtheorem{prop}{Proposition}
\newtheorem{thm}{Theorem}
\newtheorem{lem}{Lemma}

\newtheorem{fact}{Fact}

\theoremstyle{definition}

 \theoremstyle{remark}
 \newtheorem{rem}{Remark}

\title{
    Universal consistency and minimax rates for online Mondrian Forests
}

\author{
  Jaouad Mourtada
  \\
  Centre de Math\'ematiques Appliqu\'ees\\
  \'Ecole Polytechnique,
  Palaiseau, France\\
  \texttt{jaouad.mourtada@polytechnique.edu} \\
  \And
  Stéphane Gaïffas  \\
  Centre de Math\'ematiques Appliqu\'ees\\
  \'Ecole Polytechnique,
  Palaiseau, France\\
  \texttt{st\'ephane.gaiffas@polytechnique.edu}\\
  \AND
  Erwan Scornet \\
    Centre de Math\'ematiques Appliqu\'ees\\
  \'Ecole Polytechnique,
  Palaiseau, France\\
  \texttt{erwan.scornet@polytechnique.edu}
}

\begin{document}

\maketitle

\begin{abstract}
  We establish the consistency of an algorithm of Mondrian Forests~\cite{lakshminarayanan2014mondrianforests,lakshminarayanan2016mondrianuncertainty}, a randomized classification algorithm that can be implemented online. 
  First, we amend the original Mondrian Forest algorithm proposed in~\cite{lakshminarayanan2014mondrianforests}, that considers a \emph{fixed} lifetime parameter. 
  Indeed, the fact that this parameter is fixed hinders the statistical consistency of the original procedure.
 Our modified Mondrian Forest algorithm grows trees with increasing lifetime parameters $\lambda_n$, and uses an alternative updating rule, allowing to work also in an online fashion. 
 Second, we provide a theoretical analysis establishing simple conditions for consistency.
Our theoretical analysis also exhibits a surprising fact: our algorithm achieves the minimax rate (optimal rate) for the estimation of a Lipschitz regression function, which is a strong extension of previous results~\cite{arlot2014purf_bias} to an \emph{arbitrary dimension}.
\end{abstract}

\section{Introduction}
\label{sec:introduction}

Random Forests (RF) are state-of-the-art classification and regression algorithms that proceed by averaging the forecasts of  a number of
randomized
decision trees 
grown 
in parallel 
(see~\cite{breiman2001randomforests,breiman2004consistency,geurts2006extremely,biau2008consistency_rf,biau2012analysis_rf,biau2016rf_tour,denil2014narrowing,scornet2015consistency_rf}).
Despite their widespread use and remarkable success in practical applications, the theoretical properties of such algorithms are still not fully understood~\cite{biau2012analysis_rf,denil2014narrowing}.
Among these methods, \emph{purely random forests}~\cite{Breiman2000sometheory, biau2008consistency_rf,genuer2012variance_purf,arlot2014purf_bias} 
that grow the individual trees independently of the sample, are particularly amenable to theoretical analysis; the consistency of such classifiers was obtained  in~\cite{biau2008consistency_rf}.

An important limitation of the most commonly used random forests algorithms, such as Breiman's Random Forest~\cite{breiman2001randomforests} and the Extra-Trees algorithm~\cite{geurts2006extremely}, is that they are typically trained in a batch manner, using the whole dataset to build the trees.
In order to enable their use in situations when large amounts of data have to be incorporated in a streaming fashion, several online adaptations of the decision trees and RF algorithms have been proposed~\cite{domingos2000hoeffdingtree,taddy2011dynamictrees,saffari2009online-rf,denil2013online}.


Of particular interest in this article is the \emph{Mondrian Forest} algorithm, an efficient and accurate online random forest classifier~\cite{lakshminarayanan2014mondrianforests}. This algorithm is based on the Mondrian process~\cite{roy2009mondrianprocess,roy2011phd}, a natural probability distribution on 
the set of recursive 
partitions of the unit cube $[0,1]^d$.
An appealing property of Mondrian processes is that they can be updated in an online fashion: in~\cite{lakshminarayanan2014mondrianforests}, the use of the \emph{conditional Mondrian} process enabled 
to design an online algorithm that matched its batch counterpart.
%
While Mondrian Forest offer several advantages, both computational and in terms of predictive performance, the algorithm proposed in~\cite{lakshminarayanan2014mondrianforests} depends on a \emph{fixed} lifetime parameter $\lambda$ that guides the complexity of the trees. 
Since this parameter has to be set in advance, the resulting algorithm is inconsistent, as the complexity of the randomized trees remains bounded. Furthermore, an analysis of the learning properties of Mondrian Forest -- and in particular of the influence and proper theoretical tuning of the lifetime parameter $\lambda$ -- is still lacking.

In this paper, we propose a modified online random forest algorithm based on Mondrian processes. Our algorithm retains the crucial property of the original method~\cite{lakshminarayanan2014mondrianforests} that the decision trees can be updated incrementally. However, contrary to the original approach, our algorithm uses an increasing sequence of lifetime parameters $(\lambda_n)_{n \geq 1}$, so that the corresponding trees are increasingly complex, and involves an alternative online updating algorithm.
We study such classification rules theoretically, establishing simple conditions on the sequence $(\lambda_n)_{n\geq 1}$ to achieve consistency, see Theorem~\ref{thm:consistency-mondrian} from Section~\ref{sec:consistency} below.

In fact, Mondrian Forests achieve much more than what they were designed for: while they were primarily introduced to derive an online algorithm, we show in Theorem~\ref{thm:minimax} (Section~\ref{sec:minimax}) that they actually achieve minimax convergence rates for Lipschitz conditional probability (or regression) functions in arbitrary dimension.
To the best of our knowledge, such results have only been proved for very specific purely random forests, where the covariate dimension is equal to one. 

%

\paragraph{Related work.}
While random forests were introduced in the early 2000s~\cite{breiman2001randomforests}, as noted by~\cite{denil2014narrowing} the 
theoretical analysis of these methods is outpaced by their practical use. 
The consistency of various simplified random forests algorithms is first established in~\cite{biau2008consistency_rf}, as a byproduct of the consistency of individual tree classifiers.
A recent line of research~\cite{biau2012analysis_rf,denil2014narrowing,scornet2015consistency_rf} has sought to obtain theoretical guarantees (\ie consistency) for random forests variants that more closely resembled the algorithms used in practice.
Another aspect of the theoretical study of random forests is the bias-variance analysis of simplified versions of random forests~\cite{genuer2012variance_purf,arlot2014purf_bias}, such as the \emph{purely random forests} (PRF) model that performs splits independently of the data. In particular,~\cite{genuer2012variance_purf} shows that some PRF variants achieve the minimax rate for the estimation of a Lipschitz regression functions in dimension~$1$.
Additionally, the bias-variance analysis is extended in~\cite{arlot2014purf_bias}, showing that PRF can also achieve minimax rates for $C^2$ regression functions in dimension one, and considering higher dimensional models of PRF that achieve suboptimal rates.

Starting with~\cite{saffari2009online-rf}, online variants of the random forests algorithm have been considered. In~\cite{denil2013online}, the authors propose an online random forest algorithm and prove its consistency. 
The procedure 
relies on a partitioning of the data into two streams: a \emph{structure stream} (used to grow the tree structure) and an \emph{estimation stream} (used to compute the prediction in each leaf).
This separation of the data into separate streams is a way of simplifying the proof of consistency, but leads to a non-realistic setting in practice.
 
A major development in the design of online random forests is the introduction of the \emph{Mondrian Forest} (MF) classifier~\cite{lakshminarayanan2014mondrianforests,lakshminarayanan2016mondrianuncertainty}. 
This algorithm makes an elegant use of the \emph{Mondrian Process}, introduced in~\cite{roy2009mondrianprocess}, see also~\cite{roy2011phd,orbanz2015exchangeable}, to draw random trees. Indeed, this process  provides a very convenient probability distribution over the set of recursive, tree-based partitions of the hypercube.
In~\cite{balog2016mondriankernel}, the links between the Mondrian process and the Laplace kernel are used to design random features in order to efficiently approximate kernel ridge regression, leading to the so-called \emph{Mondrian kernel} algorithm.

Our approach differs from the original Mondrian Forest algorithm~\cite{lakshminarayanan2014mondrianforests}, since it introduces a ``dual''  construction, that works in the ``time'' domain (lifetime parameters) instead of the ``space'' domain (features range).
Indeed, in~\cite{lakshminarayanan2014mondrianforests}, the splits are selected 
using a Mondrian process on the range of previously observed features vectors,
and the online updating of the trees is enabled by 
the possibility of extending a Mondrian process to a larger cell using \emph{conditional Mondrian} processes.
Our algorithm incrementally grows the trees by extending the lifetime; the online update of the trees
exploits the
Markov property of the Mondrian process, a consequence of its formulation in terms of competing exponential clocks.


\section{Setting and notation}
\label{sec:setting-notations}

We first explain the considered setting allowing to state consistency of our procedure, and we describe and set notation for the main concepts used in the paper, namely trees, forests and partitions.

\paragraph{Considered setting.}

Assume we are given an i.i.d. sequence $(X_1, Y_1), (X_2,Y_2) \dots$
of $[0,1]^d \times \{0,1\}$-valued random variables 
that come sequentially,
such that each $(X_i,Y_i)$ has the same distribution as $(X,Y) $. This unknown distribution is characterized by the distribution $\mu$ of $X$ on $[0,1]^d$ and the conditional probability $\eta (x) = \P (Y=1 \cond X=x)$.

At each time step $n \geq 1$, we want to output a $0$-$1$-valued \emph{randomized classification rule}
$g_n ( \cdot ,Z , \D_n) : [0,1]^d \to \{ 0, 1\}$, where
$\D_n = \big( (X_1, Y_1) , \dots , (X_n, Y_n) \big)$ and $Z$ is a random variable that accounts for the randomization procedure; to simplify notation, we will generally denote $\wh g_n (x,Z ) = g_n (x, Z,\D_n)$. The quality of a randomized classifier $g_n$ is measured by its probability of error
\begin{equation}
    \label{eq:error}
    L (g_n) = \P (g_n(X, Z, \D_n) \neq Y \cond \D_n)
    = \P_{(X,Y), Z} (g_n(X, Z, \D_n) \neq Y)
  \end{equation}
  where $\P_{(X,Y),Z}$ denotes the integration with respect to $(X,Y),Z$ alone. 
  The quantity of Equation~\eqref{eq:error} is minimized by the  \emph{Bayes classifier} $g^* (x) = \bm 1_{\{\eta (x) > \frac 12\}}$, and its loss, 
  the \emph{Bayes error}, is denoted $L^* = L (g^*)$. 
  We say that a sequence of classification rules $(g_n)_{n \geq 1}$ is \emph{consistent} whenever $L (g_n) \to L^*$ in probability as $n \to \infty$.
\begin{rem}
   We restrict ourselves to binary classification, note however that our results and proofs can be extended to multi-class classification.
\end{rem}

  \paragraph{Trees and Forests.}

  The classification rules $(g_n)_{n \geq 1}$ we consider take the form of a \emph{random forest}, defined by averaging randomized tree classifiers.
  More precisely,   
  let $K \geq 1$ be a fixed number of randomized 
  classifiers $\wh g_n(x, Z_1), \dots, \wh g_n( x, Z_K)$ associated to the same randomized mechanism, 
  where the $Z_k$ are i.i.d. Set $Z^{(K)}= (Z_1, \dots, Z_K)$.
  The \emph{averaging classifier} $\wh g_{n}^{(K)}(x, Z^{(K)})$ is defined by taking the majority vote
  among the values $g_n(x, Z_k)$, $k=1,\dots, K$.

    Our individual randomized classifiers are \emph{decision trees}. 
  A decision tree $(T, \splits)$ is composed of the following components:
  \begin{itemize}
    \item A finite rooted ordered binary tree $T$, with nodes $\nodes (T)$, interior nodes $\inodes (T)$ and leaves $\leaves (T)$ (so that $\nodes (T)$ is the disjoint union of $\inodes (T)$ and $\leaves (T)$). Each interior node $\node$ has a left child $\lchild \node$ and a right child $\rchild \node$;
    \item A family of \emph{splits} $\splits = (\split_\node)_{\node \in \inodes (T)}$ at each interior node, where each split $\sigma_\node = (d_\node, \cut_\node)$ is characterized by its split dimension $d_\node \in \{ 1, \dots, d \}$ and its threshold $\cut_\node$.
  \end{itemize}
  Each randomized classifier $\wh g_n(x, Z_k)$ relies on a decision tree $T$, the random variable $Z_k$ is the random sampling of the splits $(\sigma_\eta)$ defining $T$.
  This sampling mechanism, based on the Mondrian process, is defined in Section~\ref{sec:online-mondrian}.

     We associate to $M= (T, \splits)$ a partition $(\cell_{\leaf})_{\leaf \in \leaves (T)}$ of the unit cube $[0,1]^d$, called a \emph{tree partition} (or \emph{guillotine partition}).
  For each node $\node \in \nodes (T)$, we define a hyper-rectangular region $\cell_\node$ recursively:
  \begin{itemize}
    \item The cell associated to the root of $T$ is $[0,1]^d$;
    \item For each $\node \in \inodes(T)$, we define
      \begin{equation*}
        \cell_{\lchild \node} := \{ x \in \cell_\eta : x_{d_\node} \leq \cut_{\node}  \} \quad \text{and} \quad \cell_{\rchild \node} := \cell_\node \setminus \cell_{\lchild \node}.
      \end{equation*}
  \end{itemize}
  The leaf cells $(\cell_{\leaf})_{\leaf \in \leaves (T)}$ form a partition of $[0,1]^d$ by construction. In the sequel, we will 
  identify
  a tree with splits $(T, \splits)$ with the associated tree partition $M (T, \splits )$, and 
  a node $\node \in \nodes (T)$ with the 
  cell $\cell_\node \subset [0,1]^d$.   The decision tree classifier outputs a constant prediction of the label in each leaf cell $A_\eta$ using a simple majority vote of the labels $Y_i$ ($1\leq i \leq n$) such that $X_i \in A_\eta$.

  \section{A new online Mondrian Forest algorithm}
  \label{sec:online-mondrian}

  We describe the Mondrian Process in Section~\ref{sub:process}, and recall the original Mondrian Forest procedure in Section~\ref{sec:original}. Our procedure is introduced in Section~\ref{sec:online-tree-growing}.



  \subsection{The Mondrian process}
  \label{sub:process}

  The probability distribution we consider on tree-based partitions of the unit cube $[0,1]^d$ is the Mondrian process, introduced in~\cite{roy2009mondrianprocess}. 
  Given a rectangular box $C = \prod_{j=1}^d
  [a_j, b_j]$, we denote  $|C| := \sum_{j=1}^d (b_j -a_j)$ its \emph{linear dimension}. 
  The Mondrian process distribution $\MP (\lambda, C)$ is the distribution of the random tree partition of $C$ obtained by the 
  sampling procedure $\mathtt{SampleMondrian} (\lambda, C)$ from Algorithm~\ref{alg:sample-mondrian}.
   
\begin{algorithm}[htbp]
  \caption{$\mathtt{SampleMondrian}(\lambda, C)$
    ; Samples a tree partition distributed as $\MP (\lambda, C)$.
  }
\label{alg:sample-mondrian}              
\begin{algorithmic}[1]
  \STATE \textbf{Parameters:} A rectangular box $C \subset \R^d$ and a lifetime parameter $\lambda >0$.
  \STATE \textbf{Call} $\mathtt{SplitCell} (C, \tau_C :=0, \lambda)$.
\end{algorithmic}
\end{algorithm}
  
 \begin{algorithm}[htbp]
   \caption{$\mathtt{SplitCell} (A, \tau, \lambda)$ ; Recursively split a cell $A$, starting from time $\tau$, until $\lambda$} 
\label{alg:split-cell}              
\begin{algorithmic}[1]
  \STATE \textbf{Parameters:} A cell 
  $\cell = \prod_{1\leq j \leq d} [a_j, b_j]$, 
  a starting time $\tau$ 
  and a lifetime parameter $\lambda$.
  \STATE \textbf{Sample} an exponential random variable $E_A$ with intensity $|A|$.
  \IF{$\tau + E_A \leq \lambda$}
  \STATE \textbf{Draw} at random a {split dimension} $J \in \{ 1,\dots, d \}$, with $\P (J = j) = (b_j - a_j)/ |A|$, and a {split threshold} $\cut_J$ uniformly in $[a_J, b_J]$.
  \STATE \textbf{Split} $A$ along the split $(J, \cut_J)$. 
  \STATE \textbf{Call} $\mathtt{SplitCell} (\lchild A, \tau + E_A, 
  \lambda)$ and $\mathtt{SplitCell} (\rchild A, \tau + E_A, 
  \lambda)$.
  \ELSE
  \STATE Do nothing.
  \ENDIF
\end{algorithmic}
\end{algorithm}

\subsection{Online tree growing: the original scheme}
\label{sec:original}

In order to implement an online algorithm, it is crucial to be able to ``update'' the tree partitions grown at a given time step.
The approach of the original Mondrian Forest algorithm~\cite{lakshminarayanan2014mondrianforests} uses a slightly different randomization mechanism, namely a Mondrian process supported in the range defined by the past feature points. More precisely, this modification amounts to replacing each call to $\mathtt{SplitCell} (A, \tau, \lambda)$ by a call to $\mathtt{SplitCell} (A^{\range (n)}, \tau, \lambda)$, where $A^{\range (n)}$ is the range of the feature points $X_1, \dots, X_n$ that fall in $A$ (\ie the smallest box that contains them). 

When a new training point $(X_{n+1}, Y_{n+1})$ arrives, the ranges of the training points may change. The online update of the tree partition then relies on the extension properties of the Mondrian process: given a Mondrian partition $M_1 \sim\MP (\lambda, C_1)$ on a box $C_1$, it is possible to efficiently sample a Mondrian partition $M_0 \sim \MP (\lambda, C_0)$ on a larger box $C_0 \supset C_1$ that restricts to $M_1$ on the cell $C_1$ (this is called a ``conditional Mondrian'', see~\cite{roy2009mondrianprocess}).

\begin{rem}
\label{rem:infinite_lambda}
In~\cite{lakshminarayanan2014mondrianforests} a lifetime parameter $\lambda = \infty$ is actually used in experiments, which essentially amounts to growing the trees completely, until the leaves are homogeneous. We will not analyze this variant here, but this illustrates the problem of using a fixed, finite budget $\lambda$ in advance.
\end{rem}


\subsection{Online tree growing: a dual approach}
\label{sec:online-tree-growing}


An important limitation of the original scheme
is the fact that it requires to fix the lifetime parameter $\lambda$ in advance. 
In order to obtain a consistent algorithm, it is required to grow increasingly complex trees. To achieve this, we propose to 
adopt a ``dual'' point of view: instead of using a Mondrian process with fixed lifetime on a domain that changes as new data points are added, we use a Mondrian process on a fixed domain (the cube $[0, 1]^d$) but with a varying lifetime $\lambda_n$ that grows with the sample size $n$. 
The rationale is that, as more data becomes available, the classifiers should be more complex and precise.
Since the lifetime, rather than the domain, is the parameter that guides the complexity of the trees, it should be this parameter that dynamically adapts to the amount of training data.



It turns out that in this approach, quite surprisingly, the trees can be updated incrementally, leading to an online algorithm. 
The ability to extend a tree partition $M_{\lambda_n} \sim \MP (\lambda_n, [0,1]^d)$ into a finer tree partition $M_{\lambda_{n+1}} \sim \MP (\lambda_{n+1}, [0,1]^d)$ relies on a different property of the Mondrian process, namely the fact that for $\lambda < \lambda'$, it is possible to efficiently sample a Mondrian tree partition $M_{\lambda'} \sim \MP (\lambda ', C)$ given its \emph{pruning} $M_{\lambda} \sim \MP (\lambda, C)$ at time $\lambda$ (obtained by dropping all splits of $M_{\lambda'}$ performed at a time $\tau > \lambda$).


 \begin{algorithm}[htbp]
   \caption{$\mathtt{ExtendMondrian} (M_\lambda, \lambda, \lambda')$ ; Extend $M_\lambda \sim \MP (\lambda, C)$ to $M_{\lambda'} \sim \MP (\lambda', C)$}
\label{alg:extend-mondrian}              
\begin{algorithmic}[1]
  \STATE \textbf{Parameters:} A tree partition $M_\lambda$, and lifetimes $\lambda \leq \lambda'$.
  \FOR{$A$ in $\leaves(M_\lambda)$}
  \STATE \textbf{Call} $\mathtt{SplitCell} (A, \lambda, \lambda')$
  \ENDFOR
\end{algorithmic}
\label{alg:extendmondrian}
\end{algorithm}

The procedure $\mathtt{ExtendMondrian}(M_\lambda, \lambda, \lambda')$ from Algorithm~\ref{alg:extendmondrian} extends a Mondrian tree partition $M_\lambda \sim \MP (\lambda, C)$ to a tree partition $M_{\lambda'} \sim \MP (\lambda ', C)$. 
Indeed, for each leaf cell $A$ of $M_\lambda$, the fact that $A$ is a leaf of $M_\lambda$ means that during the sampling of $M_\lambda$, the time of the next candidate split $\tau + E_A$ (where $\tau$ is the time $A$ was formed and $E_A \sim \Exp (|A|)$) was strictly larger than $\lambda$.
Now in the procedure $\mathtt{ExtendMondrian} (M_\lambda, \lambda, \lambda')$, the time of the next candidate split is $\lambda + E_A'$, where $E_A' \sim \Exp (|A|)$. This is precisely the where the trick resides: by the memory-less property of the exponential distribution, the distribution of $\tau_A + E_A$ conditionally on $E_A > \lambda - \tau_A$ is the same as that of $\lambda + E_A'$.
The procedure $\mathtt{ExtendMondrian}$ can be replaced by the following more efficient implementation:
\begin{itemize}
  \item Time of the next split of the tree is sampled as $\lambda + E_{M_\lambda}$ with $E_{M_\lambda} \sim \Exp (\sum_{\leaf \in \leaves (M_\lambda)} |\cell_\leaf|)$;
  \item Leaf to split is chosen using a top-down path from the root of the tree, where the choice between left or right child for each interior node is sampled at random, proportionally to the linear dimension of all the leaves in the subtree defined by the child.
\end{itemize}


\begin{rem}
  While we consider Mondrian partitions on the fixed domain
  $[0, 1]^d$, our increasing lifetime trick can be used \emph{in conjunction} with a varying domain based on the range of the data (as in the original MF algorithm), simply by applying $\mathtt{ExtendMondrian} (M_{\lambda_n}, \lambda_n, \lambda_{n+1})$ after having extended the Mondrian to the new range. In order to keep the analysis tractable and avoid unnecessary complications in the analysis, we will study the procedure on a fixed domain only.
\end{rem}

Given an increasing sequence $(\lambda_n)_{n \geq 1}$ of lifetime parameters, our modified MF algorithm incrementally updates the trees $M_\lambda^{(k)}$ for $k= 1, \dots, K$ by calling $\mathtt{ExtendMondrian} (M_{\lambda_n}^{(k)} , \lambda_n , \lambda_{n+1})$, and combines the forecasts of the given trees, as explained in Algorithm~\ref{alg:mondrian-forest}.

 \begin{algorithm}
   \caption{$\mathtt{MondrianForest} (K, (\lambda_n)_{n\geq 1})$ ;
     Trains a Mondrian Forest classifier.}
\label{alg:mondrian-forest}              
\begin{algorithmic}[1]
  \STATE \textbf{Parameters:}
  The number of trees $K$ and the lifetime sequence $(\lambda_n)_{n\geq 1}$.
  \STATE \textbf{Initialization:}
  Start with $K$ trivial partitions $M_{\lambda_0}^{(k)}$, $\lambda_0 :=0$, $k = 1, \dots, K$. Set the counts of the training labels in each cell to $0$, and the labels e.g. to $0$.
  \FOR{$n = 1, 2, \dots$}
  \STATE \textbf{Receive} the training point $(X_{n}, Y_{n})$.
  \FOR{$k = 1, \dots, K$}
  \STATE \textbf{Update} the counts of $0$ and $1$ (depending on $Y_n$) in the leaf cell of $X_n$ in $M_{\lambda_n}$.
  \STATE \textbf{Call} $\mathtt{ExtendMondrian} (M_{\lambda_{n-1}}^{(k)}, \lambda_{n-1}, \lambda_n)$.
  \STATE \textbf{Fit} the newly created leaves.
  \ENDFOR
  \ENDFOR
\end{algorithmic}
\end{algorithm}

For the prediction of the label given a new feature vector, our algorithm uses a majority vote over the predictions given by all $K$ trees.
However, other choices are possible. For instance, the original Mondrian Forest algorithm~\cite{lakshminarayanan2014mondrianforests} places a hierarchical Bayesian prior over the label distribution on each node of the tree, and performs approximate posterior inference using the so-called interpolated Kneser-Ney (IKN) smoothing. 
Another possibility, that will be developed in an extended version of this work, is tree expert aggregation methods, such as the Context-Tree Weighting (CTW) algorithm~\cite{willems1995context-basic,helmbold1997pruning} 
 or specialist aggregation methods~\cite{freund1997specialists} over the nodes of the tree, adapting them to increasingly complex trees.

Our modification of the original Mondrian Forest replaces 
the process of online tree growing with a fixed lifetime by a new process, that allows to increase lifetimes.
This modification not only allows to prove consistency, but more surprisingly leads to an optimal estimation procedure, in terms of minimax rates, as illustrated in Sections~\ref{sec:consistency} and~\ref{sec:minimax} below.



\section{Mondrian Forest with fixed lifetime are inconsistent}
\label{sec:inconsistency}

We state in Proposition~\ref{prop:inconsistency} the inconsistency of fixed-lifetime Mondrian Forests, such as the original algorithm~\cite{lakshminarayanan2014mondrianforests}. This negative result justifies our modified algorithm based on an increasing sequence of lifetimes $(\lambda_n)_{n \geq 1}$. 


\begin{prop}
\label{prop:inconsistency}
The Mondrian Forest algorithm (Algorithm~\ref{alg:mondrian-forest}) with a fixed
lifetime sequence $\lambda_n = \lambda$ is inconsistent: there exists a distribution of $(X, Y) \in [0, 1] \times \{ 0, 1 \}$ such that $L^* = 0$ and $L (g_n) = \P (g_n (X) \neq Y)$ does not tend to $0$.
This result also holds true for the original Mondrian Forest algorithm with lifetime $\lambda$.
\end{prop}

Proposition~\ref{prop:inconsistency} is established in Appendix~\ref{ap:inconsistency}.
The proof uses a result of independent interest (Lemma~\ref{lem:range}), which states that asymptotically over the sample size, for fixed $\lambda$, the restricted domain does not affect the randomization procedure.



  \section{Consistency of Mondrian Forest with lifetime sequence $(\lambda_n)$}
  \label{sec:consistency}

  The consistency of the Mondrian Forest used with a properly tuned sequence $(\lambda_n)$ is established in Theorem~\ref{thm:consistency-mondrian} below.
  
  \begin{thm}
    \label{thm:consistency-mondrian}
    Assume that $\lambda_n \to \infty$ and that ${\lambda_n^d}/{n} \to 0$. Then, the online Mondrian Forest described in Algorithm~\ref{alg:mondrian-forest} is consistent.
  \end{thm}

  This consistency result is universal, in the sense that it makes no assumption on the distribution of $X$ nor on the conditional probability $\eta$. 
  This contrasts with some consistency results on Random forests, such as Theorem~1 of~\cite{denil2013online}, which assumes that the density of $X$ is bounded by above and below.

  Theorem~\ref{thm:consistency-mondrian} does not require an assumption on $K$ (number of trees). 
  It is well-known for batch Random Forests that this meta-parameter is not a sensitive tuning parameter, and that it suffices to choose it large enough to obtain good accuracy. 
  The only important parameter is the sequence $(\lambda_n)$, that encodes the complexity of the trees. Requiring an assumption on this meta-parameter is natural, and confirmed by the well-known fact that the tree-depth is the most important tuning parameter for batch Random Forests, see for instance~\cite{biau2016rf_tour}.

  The proof of Theorem~\ref{thm:consistency-mondrian} can be found in the supplementary material (see Appendix~\ref{ap:consistency}).
  The core of the argument lies in two lemmas describing two novel properties of Mondrian trees. Lemma~\ref{lem:diameter} below provides an upper bound of order $ O(\lambda^{-1})$ on the diameter of the cell $A_\lambda (x)$ of a Mondrian partition $M_\lambda \sim \MP (\lambda, [0, 1]^d)$. This is the key to control the bias of Mondrian Forests with lifetime sequence that tend to infinity.

  
  \begin{lem}[Cell diameter]
    \label{lem:diameter}
    Let $x \in [0,1]^d$, and let $D_\lambda (x)$ 
    be the $\ell^2$-diameter
    of the cell containing $x$ in a Mondrian partition $M_\lambda \sim \MP (\lambda, [0, 1]^d)$.
    If $\lambda \to \infty$, then $D_\lambda (x) \to 0$ in probability. More precisely, for every $\delta, \lambda >0$, we have
    \begin{equation}
      \label{eq:diameter-bound}
      \P ( D_\lambda (x) \geq \delta )
      \leq
      d \left( 1 + \frac{\lambda \delta}{\sqrt{d}} \right)      
      \exp \left( - \frac{\lambda \delta}{\sqrt{d}} \right)
    \end{equation}
    and
    \begin{equation}
      \label{eq:diameter-square}
      \E \big[ D_{\lambda} (x)^2 \big] \leq \frac{4d}{\lambda^2} \, .
    \end{equation}
  \end{lem}
    The proof of Lemma~\ref{lem:diameter} is provided in the supplementary material (see Appendix~\ref{ap:diameter}). 
    The second important property needed to carry out the analysis is stated in Lemma~\ref{lem:number-splits} and helps to control the ``variance'' of Mondrian forests. 
    It consists in an upper bound of order $O (\lambda ^d)$ on the total number of splits performed by a Mondrian partition $M_\lambda \sim \MP (\lambda, [0, 1]^d) $. 
    This ensures that enough data points fall in each cell of the tree, so that the labels of the tree are well estimated. 
    The proof of Lemma~\ref{lem:number-splits} is to be found in the supplementary material (see Appendix~\ref{ap:number-splits}).


    \begin{lem}[Number of splits]
    \label{lem:number-splits}
    If $K_\lambda$ denotes the number of splits performed by a Mondrian tree partition $M_\lambda \sim \MP (\lambda, [0,1]^d)$, 
    we have 
    $\E (K_\lambda) \leq (e (\lambda + 1))^d$.
  \end{lem}

  \begin{rem}
    It is worth noting that controlling the total number of splits ensures that the cell $A_{\lambda_n} (X)$ in which a new random $X \sim \mu$ ends up contains enough training points among $X_1, \dots, X_n$ (see~Lemma~\ref{lem:number-cells} in appendix~\ref{ap:consistency}). 
    This enables to get a distribution-free consistency result.
    Another approach consists in lower-bounding the volume
     $V_{\lambda_n} (x)$ of $A_{\lambda_n} (x)$ in probability for any $x\in [0, 1]^d$, which shows that the cell $A_{\lambda_n} (x)$ contains enough training points, but this would require the extra assumption that the density of $X$ is lower-bounded.
  \end{rem}

  Remarkably, owing to the nice restriction
  properties of the Mondrian process,
  Lemmas~\ref{lem:diameter} and~\ref{lem:number-splits}
  essentially provide matching upper and lower bounds on the complexity of the partition. Indeed, in order to partition the cube $[0, 1]^d$ in cells of diameter $O (1/\lambda)$, at least $\Theta (\lambda^d)$ cells are needed; Lemma~\ref{lem:number-splits} shows that the Mondrian partition in fact contains only $O(\lambda^d)$ cells.

%




  \section{Minimax rates over the class of Lipschitz functions}
  \label{sec:minimax}

  The estimates obtained in Lemmas~\ref{lem:diameter} and~\ref{lem:number-splits} are quite explicit and sharp in their dependency on $\lambda$, and allow to study the convergence rate of our algorithm.
  Indeed, it turns out that our modified Mondrian Forest, when properly tuned, can achieve the minimax rate in classification over the class of Lipschitz functions (see e.g.~Chapter I.3 in~\cite{nemirovski2000nonparametric} for details on minimax rates).
  We provide two results: a convergence rate for the estimation of the conditional probabilities, measured by the quadratic risk, see~Theorem~\ref{thm:minimax}, and a control on the distance between the classification error of our classifier and the Bayes error, see Theorem~\ref{thm:minimax-classification}. 
  We provide also similar minimax bounds for the regression setting instead of the classification one in the supplementary material, see Proposition~\ref{prop:minimax-regression} in Appendix~\ref{ap:minimax}. 

  Let $\wh \eta_n $ be the estimate of the conditional probability $\eta$ based on the Mondrian Forest (see Algorithm~\ref{alg:mondrian-forest}) in which: 
  \begin{itemize}
  	\item[$(i)$] Each leaf label is computed as the proportion of $1$ in the corresponding leaf;
  	\item[$(ii)$] Forest prediction results from the average of tree estimates instead of a majority vote. 
  \end{itemize}

    \begin{thm}
    \label{thm:minimax}
    Assume that the conditional probability function $\eta : [0,1]^d \to [0, 1]$ is Lipschitz on $[0, 1]^d$. 
    Let $\wh \eta_n$ be a Mondrian Forest as defined in Points~(i) and~(ii), with a lifetimes sequence that satisfies $\lambda_n \asymp n^{1/(d+2)}$.
    Then, the following upper bound holds
    \begin{equation}
      \label{eq:minimax-eta-rates}
      \E (\eta (X) - \wh \eta_n (X))^2 = O (n^{-2/(d+2)})
    \end{equation}
    for $n$ large enough, which correspond to the minimax rate over the set of Lipschitz functions.
  \end{thm}
  To the best of our knowledge, Theorem~\ref{thm:minimax} is the first to exhibit the fact that a classification method based on a purely random forest can be minimax optimal in an arbitrary dimension.
  The same kind of result is stated for regression estimation in the supplementary material (see Proposition~\ref{prop:minimax-regression} in Appendix~\ref{ap:minimax}).

  Minimax rates, but only for $d=1$, were obtained in~\cite{genuer2012variance_purf,arlot2014purf_bias} for models of purely random forests such as Toy-PRF (where the individual partitions corresponded to randomly shifts of the
  regular partition of $[0, 1]$ in $k$ intervals) and PURF (Purely Uniformly Random Forests, where the partitions were obtained by drawing $k$ random thresholds at random in $[0, 1]$).

  However, for $d=1$, tree partitions reduce to partitions of $[0, 1]$ in intervals, and do not possess the recursive structure that appears in higher dimensions and makes their precise analysis difficult. 
  For this reason, the analysis of purely random forests for $d > 1$ has typically produced sub-optimal results: for example, \cite{biau2008consistency_rf} show consistency for UBPRF (Unbalanced Purely Random Forests, that perform a fixed number of splits and randomly choose a leaf to split at each step), but with no rate of convergence. 
  A further step was made by~\cite{arlot2014purf_bias}, who studied the BPRF (Balanced Purely Random Forests algorithm, where all leaves were split, so that the resulting tree was complete), and obtained suboptimal rates.
  In our approach, the convenient properties of the Mondrian process enable to bypass the inherent difficulties met in previous attempts, thanks to its recursive structure, and allow to obtain the minimax rate with transparent proof.

  Now, note that the Mondrian Forest classifier corresponds to the plugin classifier $\wh g_n (x) = \bm 1_{ \{ \wh \eta_n (x) > 1/2 \} }$, where $\wh \eta_n $ is defined in Points~(i) and~(ii).
  A general theorem (Theorem~6.5 in \cite{devroye1996ptpr}) allows us to derive upper bounds on the distance between the classification error of $\wh g_n$ and the Bayes error, thanks to
  Theorem~\ref{thm:minimax}.
  \begin{thm}
    \label{thm:minimax-classification}
    Under the same assumptions as in Theorem~\ref{thm:minimax},
    the Mondrian Forest classifier $\wh g_{n}$ with lifetime sequence $\lambda_n \asymp n^{1/(d+2)}$ satisfies
    \begin{equation}
      \label{eq:minimax-error-rates}
      L (\wh g_n) - L^* = o (n^{-1/(d+2)}).
    \end{equation}
  \end{thm}
    The rate of convergence $o(n^{-1/(d+2)})$ for the error probability with a Lipschitz conditional probability $\eta$ turns out to be optimal, 
    as shown by~\cite{yang1999minimax}. Note that faster rates can be achieved in classification under low noise assumptions such as the \emph{margin assumption}~\cite{mammen1999margin} (see e.g.~\cite{tsybakov2004aggregation,audibert2007plugin,lecue2007lownoise}).
    Such specializations of our results are to be considered in a future work, the aim of the present paper being an emphasis on the appealing optimal properties of our modified Mondrian Forest.

    \section{Experiments}    
    \label{sec:experiments}    

    We now turn to the empirical evaluation of our algorithm, and examine its predictive performance (test error) as a function of the training size.
    More precisely, we compare the modified Mondrian Forest algorithm (Algorithm~\ref{alg:mondrian-forest}) 
    to batch (Breiman RF \cite{breiman2001randomforests}, Extra-Trees-1 \cite{geurts2006extremely}) and online (the Mondrian Forest algorithm \cite{lakshminarayanan2014mondrianforests} with fixed lifetime parameter $\lambda$) Random Forests algorithms.    
    We compare the prediction accuracy (on the test set) of the aforementioned algorithms trained on varying fractions of the training data from $10 \%$ to $100\%$.

    Regarding our choice of competitors, we note that Breiman's RF is well-established and known to achieve state-of-the-art performance.
    We also included the \emph{Extra-Trees-$1$} (ERT-$1$) algorithm \cite{geurts2006extremely}, which is most comparable to the Mondrian Forest classifier since it also draws splits randomly (we note that the ERT-$k$ algorithm \cite{geurts2006extremely} with the default tuning $k = \sqrt{d}$ in the \texttt{scikit-learn} implementation \cite{pedregosa2011scikit-learn} achieves scores very close to those of Breiman's RF).
    
    In the case of online Mondrian Forests, we included our modified Mondrian Forest classifier with an increasing lifetime parameter $\lambda_n = n^{1/(d+2)}$ tuned according to the theoretical analysis (see Theorem~\ref{thm:minimax-classification}), as well as a Mondrian Forest classifier with constant lifetime parameter $\lambda = 2$.
    Note that while a higher choice of $\lambda$ would have resulted in a performance closer to that of the modified version (with increasing $\lambda_n$), our inconsistency result (Proposition~\ref{prop:inconsistency}) shows that its error would eventually stagnate given more training samples. 
    In both cases, the splits are drawn within the range of the training feature, as in the original Mondrian Forest algorithm.
    Our results are reported in Figure~\ref{fig:experiments-datasets}.

    \begin{figure}[h]
      \centering      
      \hfill
      \includegraphics[width=.45\linewidth] {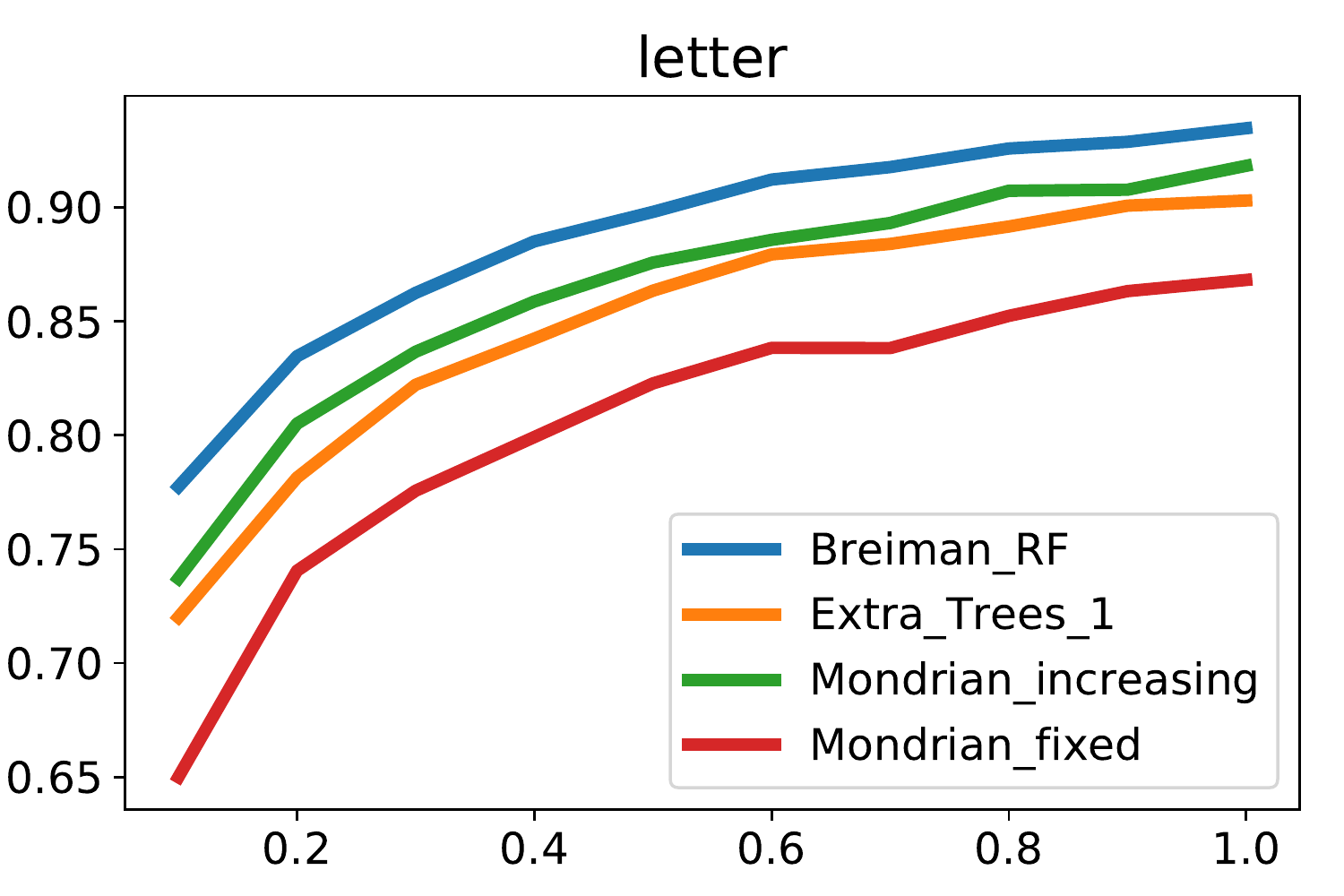}
      \hfill
      \includegraphics[width=.45\linewidth] {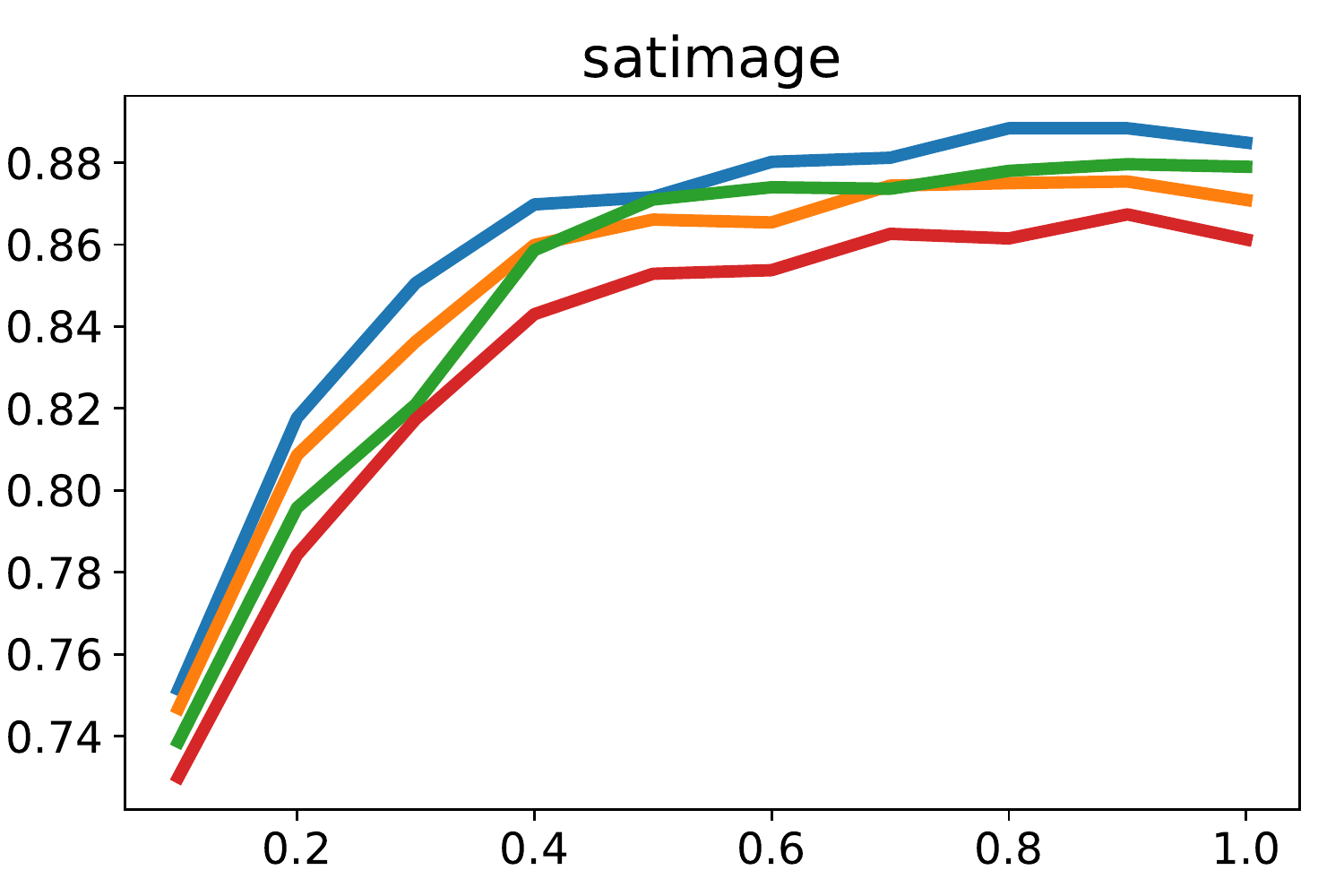}
      \hfill \textcolor{white}{.} \\
      \hfill
      \includegraphics[width=.45\linewidth] {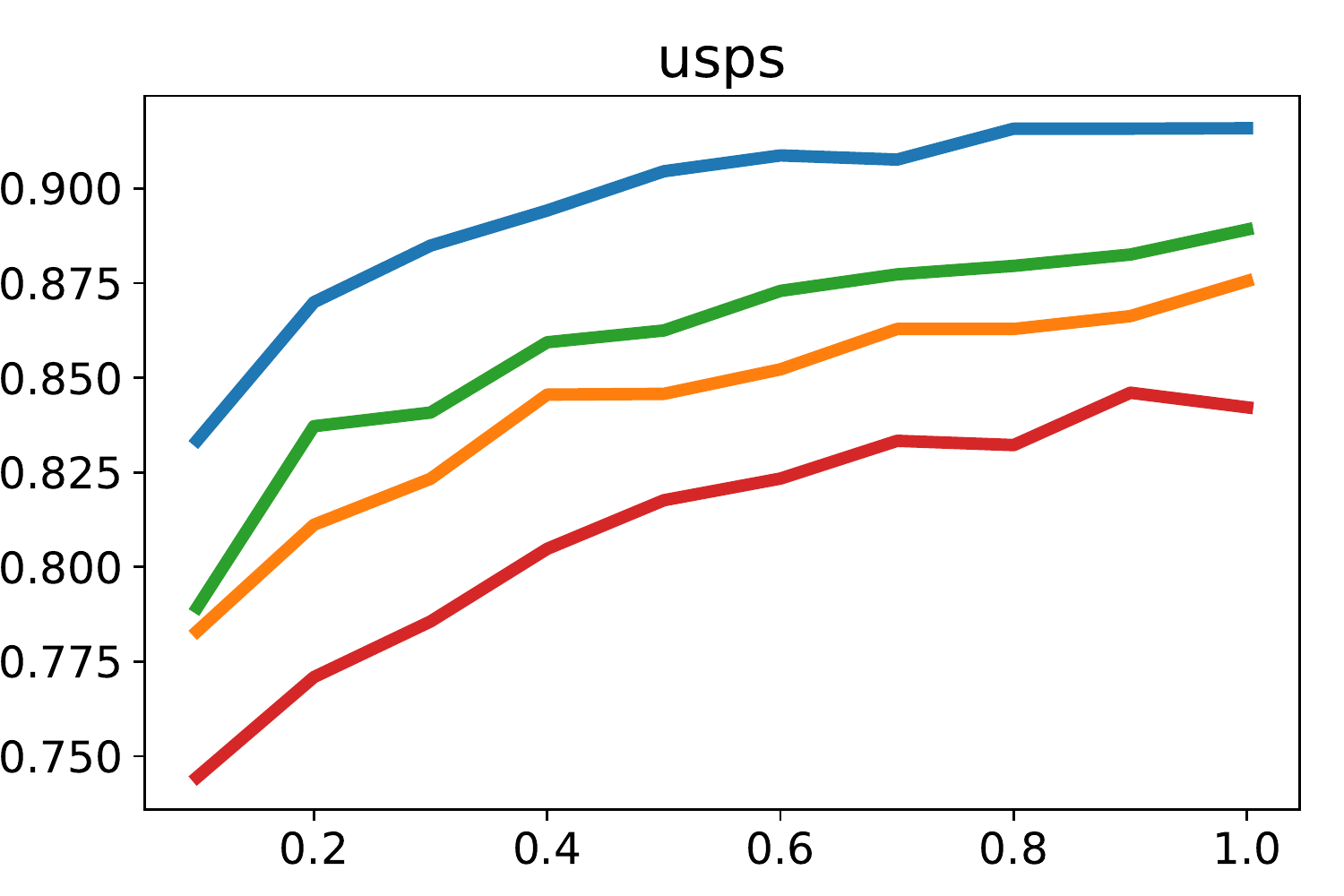}
      \hfill
      \includegraphics[width=.45\linewidth] {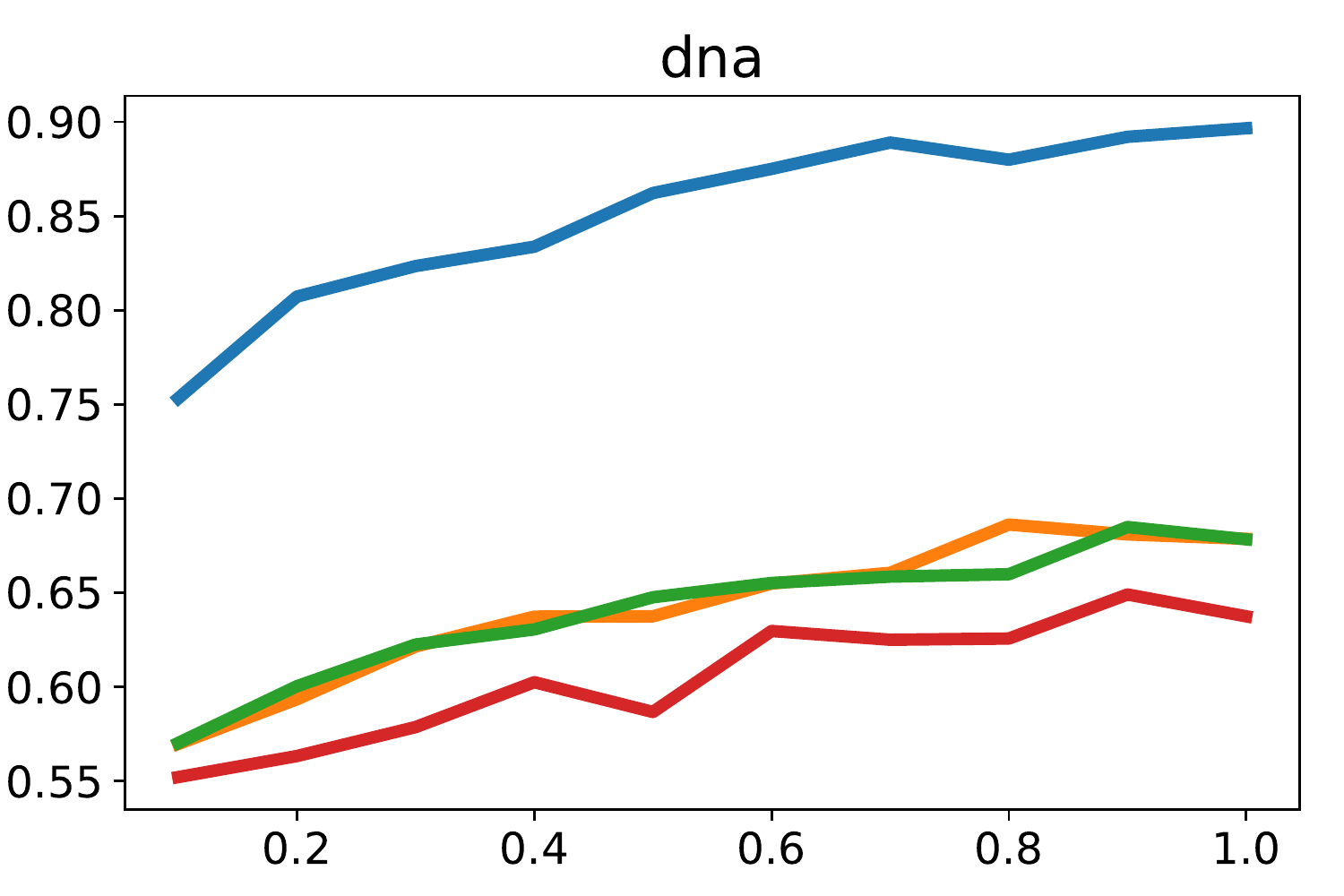}
      \hfill 
      \caption{Prediction accuracy as a function of the fraction of data used on several datasets. Modified MF (Algorithm~\ref{alg:mondrian-forest}) outperforms MF with a constant lifetime, and is better than the batch ERT-$1$ algorithm.
        It also performs almost as well as Breiman's RF (a batch algorithm that uses the whole training dataset in order to choose each split) on several datasets, while being incremental and much faster to train.
      On the \emph{dna} dataset, as noted in \cite{lakshminarayanan2014mondrianforests}, Breiman's RF outperforms the other algorithms because of the presence of a large number of irrelevant features.}
      \label{fig:experiments-datasets}
    \end{figure}

  \section{Conclusion and future work}
  \label{sec:conclusion}

  Despite their widespread use in practice, the theoretical understanding of Random Forests is still incomplete. 
  In this work, we show that amending the Mondrian Forest classifier, originally introduced to provide an efficient online algorithm, leads to an algorithm that is not only consistent, but in fact minimax optimal for Lipschitz conditional probabilities in arbitrary dimension. This new result suggests promising improvements in the understanding of random forests methods.

  A first, natural extension of our results, that will be addressed in a future work, is the study of the rates for smoother regression functions. 
  Indeed, we conjecture that through a more refined study of the local properties of the Mondrian partitions, it is possible to describe exactly the distribution of the cell of a given point. In the spirit of the work of~\cite{arlot2014purf_bias} in dimension one, this could be used to show improved rates for the bias of forests (e.g. for $C^2$ regression functions) compared to the tree bias, and hence give some theoretical insight to the empirically well-known fact that a forest performs better than individual trees.

  Second, the optimal upper bound $O(n^{-1/(d+2)})$ 
  obtained in this paper is very slow when the number of features $d$ is large.
  This comes from the well-known curse of dimensionality phenomenon, a problem affecting all fully nonparametric algorithms.
  A standard technique used in high-dimensional settings is to work under a sparsity assumption, where only $s \ll d$ features are informative (\ie affect the distribution of $Y$). In such settings, a natural strategy is to select the splits using the labels $Y_1, \dots, Y_n$, as most variants of Random Forests used in practice do. For example, it would be interesting to combine a Mondrian process-based randomization with a choice of the best split among several candidates, as performed by the Extra-Tree algorithm~\cite{geurts2006extremely}. Since the Mondrian Forest guarantees minimax rates, we conjecture that it should improve feature selection of  batch random forest methods, and improve the underlying randomization mechanism of these algorithms.
  From a theoretical perspective, it could be interesting to see how the minimax rates obtained here could be coupled with results on the ability of forests to select informative variables, see for instance~\cite{scornet2015consistency_rf}.

  








\newpage  
\appendix

\begin{center}
  \Large Supplementary material for \textbf{Universal consistency and minimax rates for online Mondrian Forests} \\
  \medskip
  {\normalsize J. Mourtada, S. Gaïffas and E. Scornet}
\end{center}

\medskip

\section{
  Proof of Lemma~\ref{lem:diameter}: diameter of the cells
}
\label{ap:diameter}

We start by recalling some important properties of the Mondrian process%
, which are exposed in~\cite{roy2009mondrianprocess}.

\begin{fact}[Consistency, Mondrian slices]
  \label{fac:mondrian-restriction}
  Let $M_\lambda \sim \MP (\lambda, [0,1]^d)$ be a Mondrian partition, and $C = \prod_{j=1}^d [a_j, b_j] \subset [0,1]^d$, be an axis-aligned box (we authorize lower-dimensional boxes when $a_j = b_j$ for some dimensions $j$). Consider the \emph{restriction} $M_\lambda|_C$ of $M_\lambda$ on $C$, \ie the partition on $C$ induced by the partition $M_\lambda$ of $[0,1]^d$. Then $M_\lambda|_C \sim \MP (\lambda, C)$.
\end{fact}

\begin{fact}[Dimension $1$]
  \label{fac:mondrian-poisson}
  For $d = 1$, 
  the splits from a Mondrian process $M_\lambda \sim \MP (\lambda, [0,1])$ form a subset of $[0, 1]$, which is distributed as a Poisson point process of intensity $\lambda dx$.
\end{fact}

We will now establish the technical lemma~\ref{lem:diameter}. 
In what follows, $x \in [0,1]^d$ is arbitrary, and we let
$A_\lambda (x)$ denote the (random) cell of a Mondrian partition $M_\lambda \sim \MP (\lambda, [0,1]^d)$ containing $x$.

\begin{proof}[Proof of Lemma~\ref{lem:diameter}]
   Let $A_\lambda (x) = \prod_{j=1}^d [L_\lambda^j(x), R_\lambda^j(x)]$ denote the (random) cell of a Mondrian partition $M_\lambda \sim \MP (\lambda, [0,1]^d)$ containing $x \in [0,1]^d$.
   By definition, the $\ell^\infty$-diameter $D_\lambda (x)$ of $A_\lambda (x)$ is $\max_{1\leq j \leq d} (R_\lambda^j (x) - L_\lambda^j (x))$. Since the random variables $R_\lambda^j (x) - L_\lambda^j (x)$, $1\leq j \leq d$, all have the same distribution (by symmetry of the definition of the Mondrian process with respect to the dimension), it suffices to consider $D_\lambda^1 (x) := R_\lambda^1 (x) - L_\lambda^1 (x)$.

   Consider the segment
   $I^1 (x) = [0,1] \times \{ (x_j)_{2\leq j \leq d} \} \simeq [0,1]$ (through the natural identification) containing $x = (x_j)_{1 \leq j \leq d}$, and denote $\Phi^1_\lambda (x) \subset [0,1]$ the restriction of $M_\lambda$ to $I^1 (x)$.
   Note that $R_\lambda^1 (x)$ (resp. $L_\lambda^1(x)$) is the lowest element of $\Phi_\lambda^1 (x)$ that is larger than $x_1$ (resp. the highest element of $\Phi_\lambda^1 (x)$ that is smaller than $x_1$), and is equal to $1$ (resp. $0$) if $\Phi_\lambda^1(x) \cap [x_1, 1]$ (resp. $\Phi_\lambda^1(x) \cap [0, x_1]$) is empty. By the facts~\ref{fac:mondrian-restriction} and~\ref{fac:mondrian-poisson}, $\Phi_\lambda (x)$ is a Poisson point process of intensity $\lambda $.

   Now, note that
   the characterization of $L_\lambda^1(x)$ and $R_\lambda^1(x)$ in terms of $\Phi_\lambda^1(x)$ (a Poisson process on $[0,1]$) implies the following: the distribution of $(L_\lambda^1(x) , R_\lambda^1(x))$ is the same as that of $(\wt L_\lambda^1(x) \vee 0 , \wt R_\lambda^1 (x) \wedge 1)$, where $\wt \Phi_\lambda^1 (x)$ is a Poisson process on $\R$ of intensity $\lambda$,  and $\wt L_\lambda^1 (x) = \sup (\wt \Phi_\lambda^1 (x) \cap (-\infty, x])$, $\wt R_\lambda^1 (x) = \inf (\wt \Phi_\lambda^1 (x) \cap [x, + \infty))$.
   By the properties of the Poisson point process, this implies that
   $(R_\lambda^1 (x) - x_1 , x_1 - L_\lambda^1(x)) \overset{d}{=} (E_1 \wedge (1-x_1) , E_2 \wedge x_1)$, where $E_1, E_2$ are independent exponential random variables with parameter $\lambda$. In particular, $D_\lambda^1 (x)
   = R_\lambda^1 (x) - x_1 + x_1 - L_\lambda^1(x)$ is stochastically upper bounded by $E_1 + E_2 \sim \Gamma (2, \lambda)$, 
   so that 
   we have for every $\delta >0$:
   \begin{equation}
     \label{eq:diameter-proof-1}
     \P (D_\lambda^1 (x) \geq \delta)
     \leq (1 + \lambda \delta ) e^{- \lambda \delta}     
   \end{equation}
   (with equality if $\delta \leq x_1 \wedge (1-x_1)$), 
   and $\E [D_\lambda^1(x)^2] \leq \E (E_1^2) + \E (E_2^2) = \frac{4}{\lambda^2}$.
      %
   Finally, the bound~\eqref{eq:diameter-bound} for the diameter $D_\lambda (x) = \sqrt{\sum_{j=1}^d D_\lambda^j(x)^2}$ follows from the observation that $\P (D_\lambda (x) \geq \delta) \leq \P (\exists j : D_\lambda^j (x) \geq \frac{\delta}{\sqrt{d}}) \leq d \, \P (D_\lambda^1 (x) \geq \frac{\delta}{\sqrt{d}}) $ and inequality~\eqref{eq:diameter-proof-1}; the bound~\eqref{eq:diameter-square} is obtained by noting that $\E [D_\lambda(x)^2] = d \,  \E [D_\lambda^1(x)^2] \leq \frac{4d}{\lambda^2}$. 
\end{proof}

\section{Proof of Lemma~\ref{lem:number-splits}: number of splits}
\label{ap:number-splits}

  \begin{proof}
    Let $A \subset \R^d$ be an arbitrary box, and let $K_\lambda^A$ denote the number of splits performed by $M_\lambda^A \sim \MP (\lambda, A)$. As shown in the proof of Proposition~3 in \cite{balog2016mondriankernel},
    since the time until a leaf $\leaf$ is split follows an exponential distribution of rate $| \cell_\leaf | \leq |A|$ (independently of the other leaves), the number of leaves $K_{t} + 1 \geq K_{t}$ 
    at time $t$ 
    is dominated by the number of individuals in a Yule process with rate $|A|$, which gives the first estimate
    \begin{equation}
      \label{eq:proof-number-splits-1}
      \E (K_\lambda^A) \leq \exp({\lambda |A|}) \, .
    \end{equation}
    This bound can be refined to the correct 
    order of magnitude in $\lambda$ in the following way. Consider the covering $\cv$ of $A$ by a regular grid of  $\lceil \lambda \rceil ^d $ boxes obtained by dividing each coordinate of $A$ in $\lceil \lambda \rceil$. Since each split of $A$ induces a split in at least one box $C \in \cv$ (\ie a split in the restriction $M_\lambda^C$ of $M_\lambda^A$ to $C$), and since $M_\lambda^C \sim \MP (\lambda, C)$ by Fact~\ref{fac:mondrian-restriction}, 
    \begin{equation}
      \label{eq:proof-number-splits-2}
      \E (K_\lambda^A)
      \leq \sum_{C \in \cv} \E (K_\lambda^C)
      \overset{(*)}{\leq} \lceil \lambda \rceil ^d \exp
      \left( \lambda \frac{|A|}{\lceil \lambda \rceil} \right)
      \leq (\lambda + 1)^d \exp ( |A|)
    \end{equation}
    where in the inequality~(*) we applied the bound~\eqref{eq:proof-number-splits-1} to every cell $C \in \cv$ (and the fact that $|C| = |A| / \lceil \lambda \rceil$). The bound of Lemma~\ref{lem:number-splits} follows by taking $A = [0, 1]^d$ in~\eqref{eq:proof-number-splits-2}.
  \end{proof}

\section{Proof of Proposition~\ref{prop:inconsistency}: original Mondrian Forests are inconsistent}
\label{ap:inconsistency}

In this appendix, we show that Mondrian Forests with fixed lifetime $\lambda$ are inconsistent, as stated in Proposition~\ref{prop:inconsistency}. We establish that this is true both for the variant based on the full domain $[0, 1]^d$, and for the original Mondrian Forests algorithm~\cite{lakshminarayanan2014mondrianforests} that restricts to the range of training data.

\subsection{
  Reduction to the full domain}
\label{ap:variable-range}

First, we begin by showing that, asymptotically, there is little difference between Mondrian trees constructed on the full domain and those restricted to the range of the training data. This is due to the fact that, as the sample size $n$ grows large, the training data will span the whole domain, as well as every cell contained in it.

\begin{lem}
\label{lem:range}
Assume
the distribution $\mu$ of $X$ satisfies: $\mu (A) \geq \alpha \vol (A)$ for every measurable $A \subset [0, 1]^d$, for some $\alpha \in (0, 1]$. Fix $\lambda >0$. For every $n \geq 1$, there exists a couple $(M_{\lambda}, M_\lambda^{\range (n)})$ such that $M_{\lambda} \sim \MP (\lambda, [0, 1]^d)$, $M_{\lambda}^{\range (n)}$ is a Mondrian partition with parameter $\lambda$ restricted to the range defined by the data points $X_1, \dots, X_n$, and $\P (M_{\lambda} = M_{\lambda}^{\range (n)}) \to 1$ as $n \to \infty$.
\end{lem}

\begin{proof}[Proof of Lemma~\ref{lem:range}]
  Let $M_{\lambda} \sim \MP (\lambda, [0, 1]^d)$ be sampled by the procedure~$\mathtt{SampleMondrian}$ (Algorithm~\ref{alg:sample-mondrian}). We will define explicitly each $M_{\lambda}^{\range (n)}$ so that they have the desired distribution, and agree with $M_{\lambda}$ on an event of high probability.

  First, consider the event $\Omega_n$
that all splits of $M_{\lambda}$ occur inside the range 
defined by the feature points among $X_1 ,  \dots, X_n$ that belong to the cell to be split. We will show that $\P (\Omega_n) \to 1$ as $n \to \infty$.
Since the tree $M_{\lambda}$ is grown independently of $(X_1, \dots, X_n)$, we may reason conditionally on $M_{\lambda}$, and $(X_1, \dots, X_n)$ remains distributed as $\mu^{\otimes n}$.
Note that $\Omega_n$ is equivalent to the following: no leaf cell of $M_{\lambda}$ contain no points among $X_1, \dots, X_n$. We can now write, denoting $\Omega_n^c$ the complementary of $\Omega_n$,
\begin{align}
  \P (\Omega^c_n \cond M_{\lambda})
  &= \P (\exists \leaf \in \leaves (M_{\lambda}) : \cell_\leaf \cap \{ X_1, \dots, X_n \} = \varnothing ) \nonumber  \\
  &\leq \sum_{\leaf \in \leaves (M_{\lambda})} \P ( \cell_\leaf \cap \{ X_1, \dots, X_n \} = \varnothing ) \nonumber \\
  &= \sum_{\leaf \in \leaves (M_{\lambda})} (1 - \mu (\cell_\leaf))^n \nonumber \\
  &\leq \sum_{\leaf \in \leaves (M_{\lambda})} (1 - \alpha \vol (\cell_\leaf))^n
    \label{eq:proof-range-hypothesis} \\
  &\mathop{\rightarrow}_{n \to \infty} 0 \quad \mbox{ a.s.}
    \label{eq:proof-range-limit}
\end{align}
where equation~\eqref{eq:proof-range-hypothesis} used the hypothesis $\mu \geq \alpha \vol$, and the convergence~\eqref{eq:proof-range-limit} is almost sure with respect to $M_\lambda$, since a.s. $\vol (A_{\leaf}) >0$ for every $\leaf \in \leaves (M_{\lambda})$. By the dominated convergence theorem (since each random variable $  \P (\Omega^c_n \cond M_{\lambda})$, $n \geq 1$, is dominated by $1$), we have $\P (\Omega_n) = \E [   \P (\Omega^c_n \cond M_{\lambda}) ] \to 0$ as $n \to \infty$.

For every $n \geq 1$, we define $M_\lambda^{\range(n)}$ as follows: on $\Omega_n^c$, we let $M_{\lambda}^{\range(n)}$ be a random Mondrian partition of lifetime $\lambda$, on the range defined by the data points $X_1, \dots, X_n$. On $\Omega_n$, we take $M_\lambda^{\range(n)}$ to be a pruning of $M_\lambda$. Specifically, for $\node \in \nodes (M_{\lambda})$, denote $E_{\node} = E_{\cell_\node} \sim \Exp (|\cell_{\node}|)$ the exponential random variables drawn during the construction of $M_{\lambda}$ (see Algorithm~\ref{alg:sample-mondrian}). Now, set $E_{\eta}^{\range (n)} := \frac{|\cell_{\node}|}{|\cell_{\node}^{\range (n)}|} E_{\eta} \sim \Exp ({|\cell_{\node}^{\range (n)}|})$, and $\tau_{\eta}^{\range (n)} := \sum_{\eta'} E_{\eta'}^{\range (n)}$, where the sum spans over the (strict) ancestors $\eta' \in \nodes (M_{\lambda})$ of $\eta$. Finally, we define $M_\lambda^{\range(n)}$ on $\Omega_n$ to be equal to the pruning of $M_{\lambda}$ obtained by keeping only the nodes $\nodes$ such that $\tau_{\eta}^{\range (n)} \leq \lambda$. By construction, $M_{\eta}^{\range (n)}$ has the distribution of a Mondrian process of parameter $\lambda$ restricted to the range of the data $X_1, \dots, X_n$.

It remains to show that $\P (M_\lambda^{\range(n)} = M_{\lambda}) \to 1$. Since we already proved that $\P (\Omega_n) \to 1$, it suffices to show that $\P (M_\lambda^{\range(n)} = M_{\lambda} \cond \Omega_n) \to 1$.


Second, 
consider the 
random variable
$\Delta_n = \sup_{\leaf \in \leaves (M_{\lambda})} \frac{|\cell_{\leaf}|}{|\cell_{\leaf}^{\range (n)}|} -1 \geq 0$. By the same argument as above, but replacing the boxes $\cell_\leaf$ ($\leaf \in \leaves (M_{\lambda})$) by interior cubes of size $\varepsilon$ around the edges of the cells $\cell_\node$ ($\node \in \nodes (M_{\lambda})$), we see that $\Delta_n \to 0$ in probability as $n \to \infty$. Since a.s. $\tau_{\leaf} < \lambda$ and $\tau_{\leaf}^{\range(n)} \leq (1 + \Delta_n) \tau_{\leaf}$ for every $\leaf \in \leaves (M_{\lambda})$, we have $\P (M_\lambda^{\range (n)} = M_{\lambda} \cond \Omega_n) \to 1$, which concludes the proof.
\end{proof}

\subsection{A simple example for fixed lifetime and range}
\label{ap:inconsistency-fixed}

In order to establish Proposition~\ref{prop:inconsistency}, it remains to provide a simple counter-example that proves the inconsistency of the Mondrian Forest algorithm for a fixed range and lifetime.

\begin{proof}
  Fix $\lambda >0$, and let $\epsilon \in (0, \frac{1}{4})$ to be specified later. 
  Let $X$ be uniformly distributed on $[0, 1]$; we set $Y = 1$ if $| X- \frac 12| \leq \epsilon $, and $0$ otherwise. Clearly, we have $L^* =0$.
  
   Denote $\wh g_{\lambda, n}^{(K)}$ the classifier described in Algorithm~\ref{alg:mondrian-forest} with $\lambda_n = \lambda$, trained on the dataset $( (X_1, Y_1), \dots, (X_n, Y_n) )$, and denote $\wh \eta_{\lambda, n}^{(K)}$ the corresponding estimate of the conditional probability $\eta$.
   Also, let $M_\lambda \sim \MP(\lambda, [0, 1]^d)$ and denote $A_\lambda (x) \subset [0, 1]$ the cell of $x \in [0, 1]$, as well as
   \[ 
     N_{\lambda, n} (x) := \sum_{i=1}^n \bm 1_{\{ X_i \in A_{\lambda} (x) \}} \ ,
     \qquad
     \wh \eta_{\lambda, n} (x) := \frac{1}{N_{\lambda, n} (x)} \sum_{i=1}^n Y_i \cdot \bm 1_{ \{ X_i \in A_\lambda (x) \} }
   \]
   (with $\wh \eta_{\lambda, n} (x) = 0$ if $N_{\lambda, n} (x) = 0$) and $\wh g_{\lambda,n} (x) := \bm 1_{ \{ \wh \eta_n (x) \geq \frac{1}{2} \} }$.
   For each $x \in [\frac{1}{2} - \epsilon, \frac{1}{2} + \epsilon]$, we have
   \begin{align*}
     \P ( \wh g_{\lambda, n}^{(K)} (x) = 1)
     = \P (\wh \eta_{\lambda, n}^{(K)} (x) \geq 1/2)
     \leq 2 \, \E [ \wh \eta_{\lambda, n}^{(K)} (x) ]
     = 2 \, \E [ \wh \eta_{\lambda, n} (x) ]     
   \end{align*}
   by Markov's inequality and the fact the $K$ trees in the forest have the same distribution as $M_\lambda$.
   Now, conditionally on $A_\lambda (x)$ and on $N_{\lambda, n}(x) = N \geq 1$, the points among $X_1, \dots, X_n$ that fall in $A_\lambda (x)$ are $N$ i.i.d. points drawn uniformly in the interval $A_\lambda (x)$, and $\wh \eta_{\lambda, n} (x)$ is just the fraction of those points that satisfy $|X_i - \frac{1}{2} | \leq \epsilon$.
   In particular,
   \begin{equation*}
     \E [\wh \eta_{\lambda, n} (x) \cond A_\lambda (x), N_{\lambda, n} (x) = N]
     = \frac{|A_\lambda (x) \cap [1/2 - \epsilon, 1/2 + \epsilon] |}{|A_\lambda (x) |}
     \leq \frac{ 2 \epsilon}{|A_\lambda (x) |}
   \end{equation*}
   so that
   \begin{equation}
     \label{eq:inconsistent-1}
     \P ( \wh g_{\lambda, n}^{(K)} (x) = 1)
     \leq 2  \epsilon \, \E [ |A_\lambda (x) |^{-1} ] \, .
   \end{equation}
   Now, recall that $M_\lambda$ is a partition of $[0, 1]$ into subintervals whose endpoints form a Poisson point process of intensity $\lambda$ (Fact~\ref{fac:mondrian-poisson}).
   In particular, a direct derivation shows that
   $\E [ |A_\lambda (x) |^{-1} ] \leq F(\lambda) := {\lambda} + 4e^{-\lambda/4} < + \infty$.
   Choosing $\epsilon := \frac{1}{4} \wedge \frac{1}{4 F(\lambda)}$ and using Equation~\eqref{eq:inconsistent-1}, we get
   $\P ( \wh g_{\lambda, n}^{(K)} (x) = 1) \leq \frac{1}{2}$.
   Finally, integrating over $X$, we get for each $n \geq 1$:
   \begin{equation}
     \label{eq:unconsistent-2}
     L ( g_n^{(K)})
     \geq \int_{1/2 - \epsilon}^{1/2 + \epsilon} \P ( \wh g_{\lambda, n}^{(K)} (x) = 0) dx
     \geq \epsilon > 0 \, ,
   \end{equation}
   so that $L ( g_n^{(K)})$ is bounded away from $0$, as announced.
%
\end{proof}

\section{Proof of Theorem~\ref{thm:consistency-mondrian}: consistency for Mondrian forests}
\label{ap:consistency}

\subsection{Some general consistency results}
\label{ap:general-consistency}

Let us recall two general consistency results that will be used in the proof.
First, the consistency of Mondrian forests can be deduced from that of the individual trees, using Proposition~\ref{prop:consistency-average}.

  \begin{prop}[Proposition~1 in \cite{biau2008consistency_rf}]
    \label{prop:consistency-average}
    If a sequence $(\wh g_n)_{n \geq 1}$ of randomized classifiers is consistent,  
    then for each $K \geq 1$, the averaged classifier $\wh g_n^{(K)}$ is consistent.
  \end{prop}

  

  Then, to establish the consistency of individual trees, 
  we use the following 
  consistency theorem for partitioning classifiers.
  %


  \begin{prop}[\cite{devroye1996ptpr}, Theorem~6.1]
    \label{prop:consistency-tree}
    Consider a sequence of randomized tree classifiers $(\wh g_n (\cdot,Z))$, grown independently of the labels $Y_1, \dots, Y_n$. For $x \in [0,1]^d$, denote $A_n (x) = A_n (x,Z)$ the cell containing $x$, $\diam A_n (X)$ its diameter, and $N_n (x) = N_n (x,Z)$ the number of input vectors among $X_1, \dots, X_n$ that fall in $A_n (x)$. Assume that, if $X$ is drawn from the distribution $\mu$:
    \begin{enumerate}
    \item $\diam A_n (X) \to 0$
       in probability, as $n \to \infty$,
     \item $N_n (X) \to \infty$
      in probability, as $n \to \infty$,
    \end{enumerate}
    Then, the tree classifier $\wh g_n$ is consistent.
  \end{prop}

  \subsection{Universal consistency}
  \label{ap:universal-consistency}


  We will need Lemma \ref{lem:number-cells} which states that the number of training observations in the cell of a point tends to infinity with $n$, if the number of splits is controlled. 

  \begin{lem}
    \label{lem:number-cells}
    Assume that the total number of splits $K_{\lambda_n}$ performed by the Mondrian tree partition $M_{\lambda_n}$ satisfies $\E (K_{\lambda_n}) / n \to 0$. Then, $N_{n} (X
    ) \to \infty$ in probability.
  \end{lem}


  \begin{proof}
    The proof extends a result in~\cite{biau2008consistency_rf} to a random number of splits. We
    fix $n \geq 1$, and
    reason conditionally on $M_{\lambda_n}$, which is by construction independent of $\D_n$ and $X$. Note that the number of leaves is  $| \leaves (M_{\lambda_n}) | = K_{\lambda_n} + 1$, and let
    $(A_\leaf)_{\leaf \in \leaves (M_{\lambda_n})}$
    be the corresponding cells. For $\leaf \in \leaves (M_{\lambda_n})$ we define $N_{\leaf}$ to be the number of points (with repetition) among $X_1, \dots, X_n, X$ that fall in the cell $A_\leaf$.
  Since $X_1, \dots, X_n, X$ are i.i.d., so that the joint distribution of $(X_1,\dots,X_n,X)$ is invariant under permutation of the $n+1$ points, conditionally on the set $S = \{ X_1, \dots, X_n, X \}$ (and on $M_{\lambda_n}$) the probability that $X$ falls in the cell $A_\leaf$ is $\frac{N_\leaf}{n+1}$. Therefore, for each $t>0$,
    \begin{align*}
      \P (N_n (X) \leq t)
      &= \E \{ \P (N_n (X) \leq t \cond S, M_{\lambda_n}) \} \\
      &= \E \left\{
          \sum_{\leaf \in \leaves (M_{\lambda_n}) \pp N_\leaf \leq t}
         \frac{N_\leaf}{n+1} \right\} \\
      &\leq \E \left\{ 
        \frac{t | \leaves (M_{\lambda_n})|}{n+1} \right\} \\
      &= \frac{t (\E (K_{\lambda_n}) +1)}{n+1} \, ,
    \end{align*}
    which tends to $0$ as $n \to \infty$ by assumption.    
  \end{proof}

%

 %


\begin{proof}[Proof of Theorem~\ref{thm:consistency-mondrian}.]  
  To prove the consistency of Mondrian forest with a lifetime sequence, we show that the two assumptions of Proposition~\ref{prop:consistency-tree} are satisfied, which proves Theorem~\ref{thm:consistency-mondrian} since our algorithm performs splits independently of the labels $Y_1, \dots, Y_n$. First, Lemma~\ref{lem:diameter} ensures that, if $\lambda_n \to \infty$, $D_{\lambda_n} (x) = \diam \cell_{\lambda_n} (x) \to 0$ in probability for every $x \in [0, 1]^d$. In particular, for every $\delta >0$,  
  $\P (\diam \cell_{\lambda_n} (X) \geq \delta)  
  = \int_{[0,1]^d} \P (\diam \cell_{\lambda_n} (x) \geq \delta) \mu (dx)  
  \to 0$ as $n \to \infty$ by the dominated convergence theorem. This establishes the first condition.

  For the second condition, Lemma~\ref{lem:number-splits} implies that $\E (K_{\lambda_n}) / n \leq e^d (\lambda_n + 1)^d /n \to 0$ by hypothesis. 
  By Lemma~\ref{lem:number-cells}, this establishes the second condition of Lemma~\ref{prop:consistency-tree}, which concludes the proof.
\end{proof}

  \section{Proof of Theorem~\ref{thm:minimax}: Minimax rates for Mondrian forests in regression}
  \label{ap:minimax}
  
  In this section, we demonstrate how the properties about Mondrian trees established in Lemmas~\ref{lem:diameter} and~\ref{lem:number-splits} imply minimax rates over the class of Lipschitz regression function, in arbitrary dimension $d$. We consider the following regression problem
  \begin{equation*}
Y = f (X) + \eps,
  \end{equation*}
where $X$ is a $[0, 1]^d$-valued random variable, $\eps$ is a real-valued random variable such that $\E (\eps \cond X) = 0$ and $\Var (\eps \cond X) \leq \sigma^2 < \infty$ a.s., and $f : [0, 1]^d \to \R$ is $L$-Lipschitz. We assume to be given $n$ i.i.d. observations $(X_1, Y_1), \dots, (X_n, Y_n)$, distributed as $(X, Y)$. We draw  $K$ i.i.d. Mondrian tree partitions $M_{\lambda_n}^{(1)}, \dots, M_{\lambda_n}^{(K)}$, distributed as $\MP (\lambda_n, [0, 1]^d)$. For all $k=1, \hdots, K$, we let $\wh f_n^{(k)}(x)$ be the $k$th Mondrian tree estimate at $x$, that is the average\footnote{With the convention that if no training point $X_i$, $1\leq i \leq n$, falls in $A_{\lambda_n} (x)$, then $\wt f_n (x) := 0$.} of the labels $Y_i$ such that $X_i$ belongs to the cell containing $x$ in the partition $M_{\lambda_n}^{(k)}$.
Finally, the Mondrian forest estimate at $x$ is given by 
$$\wh f_n = \frac 1K \sum_{k=1}^K \wh f_n^{(k)} : [0, 1]^d \to \R \,.$$

%
  \begin{prop}
    \label{prop:minimax-regression}


    The quadratic risk $R (\wh f_n) = \E (\wh f_n (X) - f(X))^2$ of $\wh f_n$ is upper bounded as follows:
    \begin{equation}
      \label{eq:risk-regression}
      R (\wh f_n) \leq \frac{4 d L^2}{\lambda_n^2}
      + \frac{1 + e^d (1+\lambda_n)^d}{n}
      \left( 2 \sigma^2 + 9 \| f \|_{\infty} \right)
    \end{equation}
    In particular, the choice $\lambda_n = n^{1/(d+2)}$ yields a risk rate $R (\wh f_n ) = O (n^{- 2/(d+2)})$.
  \end{prop}

  \begin{proof}
    First, by the convexity of the function $y \mapsto (y - f(x))^2$ for any $x\in [0, 1]^d$, we have $R (\wh f_n) \leq \frac 1K \sum_{k=1}^K R(\wh f_n^{(k)}) = R (\wh f_n^{(1)})$ since the random trees classifiers have the same distribution. Hence, it suffices to prove the risk bound~\eqref{eq:risk-regression} for a single tree; in the following, we assume that $K=1$, and consider the random estimator $\wh f_n$ associated to a tree partition $M_{\lambda_n} \sim \MP (\lambda_n, [0, 1]^d)$.

    Since the splits of the tree partition $M_{\lambda_n}$ are performed independently of the training data $(X_1, Y_1), \dots, (X_n, Y_n)$ we can write the following bias-variance decomposition of the risk for \emph{purely random forests}, first noticed by~\cite{genuer2012variance_purf}:
    \begin{equation}
      \label{eq:bias-variance}
      R (\wh f_n) = \E (f (X) - \wt f_{\lambda_n} (X))^2 +
      \E (\wt f_{\lambda_n} (X) - \wh f_{\lambda_n} (X))^2 \, ,
    \end{equation}
    where we denoted $\wt f_{\lambda_n} (x) := \E (f(X) | X \in A_{\lambda_n} (x))$ (which only depends on the random partition $M_{\lambda_n}$) for every $x$ in the support of $\mu$.
    The first term of the sum, the \emph{bias}, measures how close $f$ is to  its best approximation $\wt f_n$ that is constant on the leaves of $M_{\lambda_n}$ 
    (on average over $M_{\lambda_n}$). The second term (the \emph{variance}) measures how well the expected value $\wt f_n (x) = \E (Y \cond X \in A_{\lambda_n} (x))$ (\ie the optimal label on the leaf $A_{\lambda_n} (x)$) is estimated by the empirical average $\wh f_n (x)$, averaged over the sample $\D_n$ and the partition $M_{\lambda_n}$.
    
    The bias term is bounded as follows: for each $x\in [0, 1]^d$ in the support of $\mu$, we have
    \begin{align}
      | f (x) - \wt f_n (x) |
      &= \left| \frac{1}{\mu (A_{\lambda_n} (x))} \int_{A_{\lambda_n} (x)} ( f(x) - f(z)) \mu (dz) \right| \nonumber \\
      &\leq \sup_{z \in A_{\lambda_n} (x)} | f(x) - f(z) | \nonumber \\
      &\leq L \sup_{z \in A_{\lambda_n} (x)} \| x - z \|_{2}
        \label{eq:bound-lipschitz}
      \\
      &= L D_{\lambda_n} (x), \nonumber
    \end{align}
    where $D_{\lambda_n} (x)$ is the $\ell^2$-diameter of $A_{\lambda_n}(x)$; note that inequality~\eqref{eq:bound-lipschitz} used the assumption that $f$ is $L$-Lipschitz. By Lemma~\ref{lem:diameter}, this implies
    \begin{equation}
      \label{eq:bound-bias-1}
      \E (f (x) - \wt f_n (x))^2
      \leq L^2 \E [D_{\lambda_n} (x)^2]
      \leq \frac{4 d L^2}{\lambda_n^2} \, .      
    \end{equation}
    Integrating the bound~\eqref{eq:bound-bias-1} with respect to $\mu$ yields the following bound on the integrated bias:
    \begin{equation}
      \label{eq:bound-bias}
      \E (f(X) - \wt f_n (X))^2 \leq \frac{4 d L^2}{\lambda_n^2} \, .
    \end{equation}

    In order to bound the variance term,
    we use the following fact (\cite{arlot2014purf_bias}, Proposition~2): if $U$ is a random tree partition of the unit cube in $k+1$ cells (with $k \in \N$ deterministic) formed independently of the training data $\D_n$, we have
    \begin{equation}
      \label{eq:proof-var-1}
      \E (\wt f_U (X) - \wh f_U (X))^2
      \leq
      \frac{k + 1}{n} \left( 2 \sigma^2 + 9 \| f \|_{\infty} \right) \, .
    \end{equation}
    For every $k \in \N$, applying the upper bound~\eqref{eq:proof-var-1} to the random partition $M_{\lambda_n} \sim \MP (\lambda_n , [0, 1]^d)$ conditionally on the event $\{ K_{\lambda_n} = k \}$ (where $K_{\lambda_n}$ denotes the number of splits performed by $M_{\lambda_n}$), and summing over $k$, we get
    \begin{align}
      \label{eq:proof-var-2}
      \E (\wt f_{\lambda_n} (X) - \wh f_{\lambda_n} (X))^2
      &= \sum_{k = 0}^{+ \infty} \P (K_{\lambda_n} = k) \, \E [ (\wt f_{\lambda_n} (X) - \wh f_{\lambda_n} (X))^2 \cond K_{\lambda_n} = k] \nonumber\\
      &\leq \sum_{k = 0}^{+ \infty} \P (K_{\lambda_n} = k) \,\frac{k+1}{n}
        \left( 2 \sigma^2 + 9 \| f \|_{\infty} \right) \nonumber\\
      &= \frac{1 + \E (K_{\lambda_n}) }{n} \left( 2 \sigma^2 + 9 \| f \|_{\infty} \right). \nonumber     
    \end{align}
    Then, applying Lemma~\ref{lem:number-splits} gives an upper bound of the variance term:
    \begin{equation}
      \label{eq:bound-variance}
      \E (\wt f_{\lambda_n} (X) - \wh f_{\lambda_n} (X))^2
      \leq \frac{1 + e^d (1+\lambda_n)^d}{n}
      \left( 2 \sigma^2 + 9 \| f \|_{\infty} \right) \, .
    \end{equation}
    Combining the bounds~\eqref{eq:bound-bias} and~\eqref{eq:bound-variance} with the decomposition~\eqref{eq:bias-variance} yields the desired bound~\eqref{eq:risk-regression}.
  \end{proof}


{\small  

\bibliographystyle{plain}

  }

\end{document}